\def\year{2018}\relax
\documentclass[letterpaper]{article} 
\usepackage{aaai}  
\usepackage{times}  
\usepackage{helvet}  
\usepackage{courier}  
\usepackage{url}  
\usepackage{graphicx}  

\usepackage[ruled,linesnumbered,vlined]{algorithm2e} 
\usepackage{amsmath}
\usepackage{amsfonts}
\usepackage{amssymb}
\usepackage{bm}
\usepackage{mathtools}
\usepackage{amsthm}
\usepackage{multirow}
\usepackage{url}
\usepackage{xcolor}
\definecolor{pastelgreen}{rgb}{0.01, 0.75, 0.24}
\definecolor{red(pigment)}{rgb}{0.93, 0.11, 0.14}
\definecolor{bleudefrance}{rgb}{0.19, 0.55, 0.91}

\SetKwComment{Comment}{$\triangleright$\ }{}
\newtheorem{dfn}{Definition}
\newtheorem{thm}{Theorem}

\newtheorem{prop}{Proposition}
\theoremstyle{plain} 
\newtheorem{eg}{Example}

\newcommand\tab[1][1cm]{\hspace*{#1}}
\newcommand{\citet}[1]{\citeauthor{#1}~(\citeyear{#1})}

\nocopyright

\begin{document}
%
\title{ A Unified Framework for Planning in \\Adversarial and Cooperative Environments}
\author{ 
Anagha Kulkarni, Siddharth Srivastava and Subbarao Kambhampati \\
School of Computing, Informatics, and Decision Systems Engineering \\ Arizona State University, Tempe, AZ 85281 USA \\
\{anaghak, siddharths, rao\} @ asu.edu \\
}
\maketitle

\begin{abstract}
Users of AI systems may rely upon them to produce plans for achieving desired objectives. Such AI systems should be able to compute obfuscated plans whose execution in adversarial situations protects privacy, as well as legible plans which are easy for team members to understand in cooperative situations. We develop a unified framework that addresses these dual problems by computing plans with a desired level of comprehensibility from the point of view of a partially informed observer. For adversarial settings, our approach produces obfuscated plans with observations that are consistent with at least $\emph{k}$ goals from a set of decoy goals. By slightly varying our framework, we present an approach for goal legibility in cooperative settings which produces plans that achieve a goal while being consistent with at most $\emph{j}$ goals from a set of confounding goals. 
In addition, we show how the observability of the observer can be controlled to either obfuscate or clarify the next actions in a plan when the goal is known to the observer.
We present theoretical results on the complexity analysis of our problems.
We demonstrate the execution of obfuscated and legible plans in a cooking domain using a physical robot Fetch. 
We also provide an empirical evaluation to show the feasibility and usefulness of our approaches using IPC domains.
\end{abstract}

\section{Introduction}

AI systems have become quite ubiquitous. As users, we heavily rely on these systems to plan our day-to-day activities. Since all these systems have logging and tracking abilities, an observer can get access to our data and our actions. Such observers can be of two types: adversarial or cooperative. In adversarial settings, like mission planning, military intelligence, reconnaissance, counterintelligence, etc., protection of sensitive data can be of utmost importance to the agent. In such situations, it is necessary for an AI system to produce plans that reveal neither the intentions nor the activities of the agent. On the other hand, in case of a cooperative observer, the AI system should be able to produce plans that help clarify the intent of the agent. Therefore, it is desirable for an AI system to be capable of computing both obfuscated plans in adversarial settings and legible plans in cooperative settings. 

In this work, we propose a new unifying formalization and algorithms for computing obfuscated plans as well as legible plans. 
In our framework, we consider two agents: an acting agent and an observer. The acting agent has full observability of its activities. The observer is aware of the agent's planning model but has partial observability of the agent's activities. The observations are emitted as a side effect of the agent's activities and are received by the observer. In the following example, we illustrate the influence of an observation model on the belief space of the observer.

\begin{figure}
\label{fig:example2}
\includegraphics[width=\columnwidth]{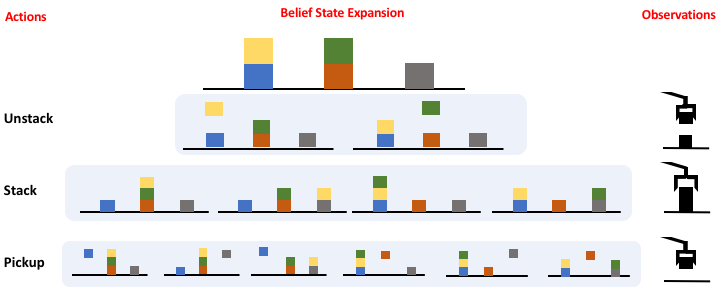}
\caption{A \texttt{blocksworld} example illustrating belief expansion due to observation equivalence induced by the observation model of the observer.}
\end{figure}

\begin{eg}
\label{example:1}
Consider a \texttt{blocksworld} planning domain, where the actions \texttt{unstack}, \texttt{stack}, \texttt{pickup} and \texttt{putdown} are available. The observer can only receive partial observations using sensors on the table and on the gripper. An observation is emitted whenever the gripper is holding a block and indicates whether the block was lifted from the table or from the top of another block. This allows our observer to learn whether or not the current action has the agent holding a block (for example in \texttt{unstack} and \texttt{pickup}, it is holding a block) and whether or not the current action involves a block that was on the table (for example, in \texttt{pickup} and \texttt{putdown}, the blocks are picked up from / putdown on  the table). Thus, our observer can distinguish between types of actions that were performed but is still unaware of which blocks were used. In Figure \ref{fig:example2}, we illustrate the progression of belief states from the observer's point of view based on the observation model defined above. 
\end{eg}

An adversarial observer may be able to use the information gleaned from observations to interfere with or hamper the agent's activities. For example, consider the keystroke timing attack \cite{song2001timing} where the observer retrieves observations about keystroke timing by studying an agent's inter-keystroke timings. 
Through such traffic analysis attacks, the observer can learn the passwords typed by an agent while connecting to a remote machine. On the other hand, in cooperative scenarios, an agent is required to communicate its intentions to the observer as quickly and clearly as possible. For example, consider a robot who is capable of assembling either chairs or tables. A chair has three components: seat, back and legs; and a table has two components: surface and legs. Whenever the robot is holding a component, the observer receives an observation regarding the type of component. In order to notify about a task of say, assembling a chair, the robot can start with the seat or the back components rather than with the legs to make its objectives clearer to the observer.

In this work, we develop a coherent set of notions for goal obfuscation and goal legibility. Our approach computes the solutions for each of these problems using the variants of a common underlying algorithm. Our approach assumes offline settings, where the observer receives the observations after the agent has finished executing a plan. In the case of a goal obfuscation problem, there exist multiple decoy goals and one true goal. The observer is unaware of the agent's true goal, and the objective is to generate a plan solution without revealing it. Our solution ensures that \emph{at least} $k$ goals are possible at the end of the observation sequence. On the other hand, in the goal legibility problem, there exist multiple confounding goals and a true goal. Here the objective is to reveal \emph{at most} $j$ goals to the observer. Our solution ensures that at most $j$ goals are possible at the end of the observation sequence. 
We also consider a variant of obfuscation and legibility where the adversary knows the goal of the agent and wants to obfuscate or reveal the next action in the plan to achieve that goal, we call these problems plan obfuscation and plan legibility respectively. For plan obfuscation, the objective is to generate a plan solution with an observation sequence that is consistent with at least $\ell$ diverse plans. On the other hand, for plan legibility, the objective is to generate a plan solution that is consistent with at least $m$ similar plans.  

In the following sections, we present a common framework that encapsulates the planning problems discussed above. And thereafter, we discuss each of the problems in detail. We also provide a theoretical and empirical analysis of the value and scope of our approaches.

\section{Controlled Observability Planning Problem}

\subsection{Classical Planning}

A classical planning problem can be defined as a tuple $\mathcal{P}= \langle \mathcal{F}, \mathcal{A} , \mathcal{I}, G \rangle $, where $\mathcal{F}$, is a set of fluents, $\mathcal{A}$, is a set of actions. A state $s$ of the world is an instantiation, $\mathcal{F}^i$ of $\mathcal{F}$. The initial state $\mathcal{I}$ is the  instantiation of all fluents in $\mathcal{F}$ and the goal $G$ is a subset of instantiated fluents in $\mathcal{F}$. 
Each action $a \in \mathcal{A}$ is a tuple of the form $\langle pre(a), add(a), delete(a), c(a) \rangle$ where $c(a)$ denotes the cost of an action, $pre(a) \subseteq \mathcal{F}$ is a set of preconditions for the action $a$, $add(a) \subseteq \mathcal{F}$ is a set of positive effects and $delete(a) \subseteq \mathcal{F}$ is a set of negative effects, i.e., $\Gamma(s, a) \models \bot$ if $s \not\models pre(a)$; else $\Gamma(s, a) \models s \cup add(a) \setminus delete(a)$ where $\Gamma(\cdot)$ is the transition function. The solution to $\mathcal{P}$ is a \emph{plan} or a sequence of actions $\pi = \langle a_1, a_2, \ldots, a_n \rangle$, such that, $\Gamma(\mathcal{I},\pi) \models G$, i.e., starting from the initial state sequentially executing the actions lands the agent in a goal state. The cost of the plan, $c(\pi)$, is summation of the cost of all the actions in the plan $\pi$, $c(\pi) = \sum_{a_i\in\pi}c(a_i)$.

\subsection{Problem Setting}

We now introduce a general planning problem framework that will be used to define adversarial and cooperative cases in the following sections. The controlled observability problem involves an acting agent and an observer. 

\begin{dfn} 
A \textbf{controlled observability planning problem} is a tuple, $\mathcal{P_{CO}} = \langle \mathcal{D}, \mathcal{G}, \Omega, \mathcal{O} \rangle$, where,

\begin{itemize}
\item $\mathcal{D} = \langle \mathcal{F}, \mathcal{A}, \mathcal{I} \rangle$ is the planning domain of the agent. 
\item $\mathcal{G} = \{G_1 \cup G_2 \ldots \cup G_{n-1} \cup G_A\}$ is a set of candidate goal conditions, each defined by subsets of fluent instantiations, where $G_A$ is the true goal of the agent.
\item $\Omega = \{o_i | i = 1, \ldots, m \}$ is a set of $m$ observations that can be emitted as a result of the action taken and the state transition. 
\item $\mathcal{O} : (\mathcal{A} \times \mathcal{S}) \rightarrow \Omega $ is a many-to-one observation function which maps the action taken and the next state reached to an observation in $\Omega$. That is to say, the observations are deterministic, each $\langle a, s' \rangle$ pair is associated with a single observation but multiple pairs can be mapped to the same observation. 
\end{itemize}
\end{dfn}

The observer has access to $\mathcal{P_{CO}}$, but is unaware of the true goal of the agent. Also, the observer does not have access to the actions performed by the agent, instead receives the observations corresponding to the plan executed by the agent. The observation function can be seen as a sensor model, as modeled in several prior works \cite{geffner2013concise,bonet2014belief,keren2016privacy}. For every action taken by the agent and an associated state transition, the observer receives an observation. This observation might be consistent with multiple action-state pairs because of the many-to-one formulation of $\mathcal{O}$. Therefore, the observer operates in the belief space. The agent takes the belief space of the observer into account in its planning process, so as to control the observability of the observer. 
Our formulation defines an offline scenario where the observer receives all the observations once the plan has been executed by the agent. 

\subsection{Observer's Belief Space}

The observer may use its observations of the agent's activity to maintain a \emph{belief state}, or the set of possible states consistent with the observations.

\begin{dfn}
\label{dfn:belief}
A \textbf{belief}, $b_i$, induced by observation, $o_i$, emitted by action, $a_i$, and resulting state, $s_i$, is, $b_i = \{ \hat{s}_i ~|~\exists \hat{a}_i, \  \mathcal{O}(\hat{a}_i, \hat{s}_i) = o_i \land \mathcal{O}(a_i, s_i) = o_i \}$.
\end{dfn}

Whenever a new action is taken by the agent, the observer's belief can be updated as follows:

\begin{dfn} A \textbf{belief update}, $b_{i+1}$ for belief $b_{i}$ is defined as, $b_{i+1} = update(b_i,  o_{i+1}) = \{ \hat{s}_{i+1}~|~\exists \hat{s}_{i}, \exists \hat{a}_{i+1},\  \Gamma(\hat{s}_{i}, \hat{a}_{i+1}) \models \hat{s}_{i+1} \land \hat{s}_{i} \in b_i \land \mathcal{O}(\hat{a}_{i+1}, \hat{s}_{i+1}) = o_{i+1}  \}$.
\end{dfn}

A sequence of such belief updates gives us the observer's belief sequence that is consistent with a sequence of observations emitted by the agent.

\begin{dfn} A \textbf{belief sequence} induced by a plan p starting at state $s_0$, BS(p, $s_0$), is defined as a sequence of beliefs $\langle b_o, b_1, \ldots, b_n \rangle$ such that there exist $o_0, o_1, o_2, \ldots, o_n \in \Omega$ where,
\begin{itemize}
\item $o_i = \mathcal{O}(a_i, s_i)$
\item $b_0 = \{\hat{s}_0 | \mathcal{O}(\emptyset, s_0) = o_0 \land \mathcal{O}(\emptyset, \hat{s}_0) = o_0\}$
\item $b_{i+1} = update(b_i, o_{i+1})$
\end{itemize}
\end{dfn}

The objective of the agent is to generate a desired belief in the observer's belief space, such that the last belief in the induced belief sequence satisfies goal conditions from the candidate goal set $\mathcal{G}$ including $G_A$. 




\subsection{Variants of $\mathcal{P_{CO}}$}

We now discuss the two major variants of $\mathcal{P_{CO}}$ namely, goal obfuscation and goal legibility planning problems. 

\subsubsection{Goal Obfuscation}

The adversary is aware of agent's candidate goal set but is unaware of agent's true goal. The aim of goal obfuscation is to hide this true goal from the observer. This is done by taking actions towards agent's true goal, such that, the corresponding observation sequence exploits the observer's belief space in order to be consistent with multiple goals. 

\begin{dfn} 
A \textbf{goal obfuscation planning problem}, is a $\mathcal{P_{CO}}$, where, $\mathcal{G} = \{G_A \cup G_1 \cup \ldots \cup G_{n-1} \}$, is the set of $n$ goals where $G_A$ is the true goal of the agent, and $G_1, \ldots, G_{n-1}$ are decoy goals. 
\end{dfn}

A solution to a goal obfuscation planning problem is a \textit{k-ambiguous} plan. The objective here is to make the observation sequence consistent with at least $k$ goals, out of which $k-1$ are decoy goals, such that, $k \leq n$. These $k-1$ goals can be chosen by the robot so as to maximize the obfuscation. 

\begin{dfn} 
\label{def:obf}
A plan, $\pi_k$, is a \textbf{k-ambiguous plan}, if $\Gamma(\mathcal{I}, \pi_k) \models G_A$ and the last belief, $b_n \in BS(\pi_k, \mathcal{I})$, satisfies the following,  $|b_n \cap \mathcal{G}| \geqslant k$, where $1 \geqslant k \geqslant n$.
\end{dfn}

\begin{dfn} An observation sequence $O_k = \langle o_1, \ldots, o_n \rangle$ is \textbf{k-ambiguous observation sequence} if it is an observation sequence emitted by a k-ambiguous plan. 
\end{dfn}

A \textit{k-ambiguous} plan achieves at least $k$ goals in the last belief of the observation sequence. 

\subsubsection{Goal Legibility}

The aim of goal legibility is to take goal-specific actions which help the observer in deducing the robot's goal. This can be useful in cooperative scenarios where the robot wants to notify the observer about its goal without explicit communication. This case is exactly opposite of the obfuscation case. 

\begin{dfn}
A \textbf{goal legibility planning problem} is a $\mathcal{P}_{CO}$, where, $\mathcal{G} = \{ G_A \cup G_1 \cup \ldots \cup G_{n-1} \}$ is the set of $n$ goals where $G_A$ is the true goal of the agent, and $G_1, \ldots, G_{n-1}$ are confounding goals. 
\end{dfn}

The objective here is to generate legible plans so as to reveal at most $j$ goals. Here we ensure that the plans are consistent with \emph{at most} $j$ goals so as to minimize the number of goals in the observer's belief space. 

\begin{dfn} 
\label{def:leg}
A plan, $\pi_j$, is a \textbf{j-legible plan}, if $\Gamma(\mathcal{I}, \pi_j) \models G_A$ and the last belief, $b_n \in BS(\pi_j, \mathcal{I})$, satisfies the following, $|b_n \cap \mathcal{G}| \leqslant j$, where $1 \geqslant j \geqslant n$.
\end{dfn}

The definition of \emph{j-legible} observation sequence follows that of \emph{k-ambiguous} case.


\subsection{Complexity Analysis}

In this section, we discuss the complexity results for $\mathcal{P_{CO}}$. Given the Definitions \ref{def:obf} and \ref{def:leg} of goal obfuscation and goal legibility plan solutions, we prove that the plan existence problem for $\mathcal{P_{CO}}$ is EXPSPACE-complete.

\begin{thm}
The plan existence problem for a controlled observability planning problem is EXPSPACE-hard. 
\end{thm}

\begin{proof}
To show that the plan existence problem for $\mathcal{P_{CO}}$ is EXPSPACE-hard, we will show that the NOD (No-Observability Deterministic) planning problem is reducible to $\mathcal{P_{CO}}$. The plan existence problem for NOD has been shown to be EXPSPACE-complete \cite{haslum1999some,rintanen2004complexity}. 


Let $\mathcal{P}_{N} = \langle \mathcal{F}_{N}, \mathcal{A}_{N}, \mathcal{I}_{N}, G_{N}, \mathcal{V} \rangle$ be a NOD planning problem, where, $\mathcal{F}_{N}$ is the set of fluents (or Boolean state variables), such that, state $s$ is an instantiation of $\mathcal{F}_{N}$. $\mathcal{A}_{N}$ is a set of actions, such that, when an action $a \in \mathcal{A}_{N}$ is applied to a state, $s_i$, a deterministic transition to the next state occurs, $\Gamma(s_{i}, a) \models s_{i+1}$. $\mathcal{I}$ and $G$ are Boolean formulae that represent sets of initial and goal states. $\mathcal{V} = \emptyset$ is the set of observable state variables. Since the underlying system state is unknown, the deterministic transition function does not reveal the hidden state. $\mathcal{P}_N$ can be expressed as a $\mathcal{P_{CO}}$ problem, $\mathcal{P}_C = \langle \mathcal{D}_C, G_C, \Omega_C, \mathcal{O}_C \rangle$, where, $\mathcal{D}_C = \{ \mathcal{F}_{C}, \mathcal{A}_{C}, \mathcal{I}_{C}\}$, such that $\mathcal{I}_{C}$ is a set of possible initial states, $G_C$ is a subset of instantiations in $\mathcal{F}_{C}$, $\Omega = \emptyset$ and $\mathcal{O} = \emptyset$.


Suppose $\pi_{\mathcal{P}_C} = \langle a_1, \ldots, a_r \rangle$ is a plan solution to $\mathcal{P}_C$, such that, $\Gamma(\mathcal{I}_C, \pi_{\mathcal{P}_C}) \models G_C$ and the last belief $b_r \in BS(\pi_{\mathcal{P}_C}, \mathcal{I}_C)$ satisfies $|b_r \cap G_C| = 1$. Then according to the definition of $\mathcal{P}_N$, the plan $\pi_{\mathcal{P}_C}$ has a last belief, such that, $\exists s_r \in b_r, s_r \models G_C$ and therefore solves $\mathcal{P}_N$.

Conversely, suppose $\pi_{\mathcal{P}_N} = \langle a_1, \ldots, a_q \rangle$ is a plan solution to $\mathcal{P}_N$, such that, $\Gamma (\mathcal{I}_N, \pi_{\mathcal{P}_N}) \models G_{N}$. Let $B_q$ be the belief associated with the last action in $\pi_{\mathcal{P}_N}$. Since it achieves the goal, we can say that $|B_q \cap G| = 1$. According to Definitions \ref{def:obf}, \ref{def:leg}, for $k=j=1$, $B_q$ satisfies the condition. Therefore $\pi_{\mathcal{P}_N}$ is a solution to $\mathcal{P}_C$. 
\end{proof}

\begin{thm}
The plan existence problem for a controlled observability planning problem is EXPSPACE-complete. 
\end{thm}

\begin{proof}

In $\mathcal{P_{CO}}$, the planner operates in belief space and the state space is bounded by $2^{2^{|\mathcal{F}|}}$, where $|\mathcal{F}|$ is the cardinality of the fluents (or Boolean state variables). If there exists a plan solution for $\mathcal{P_{CO}}$, it must be bounded by $2^{2^{|\mathcal{F}|}}$ in length. Any solution longer in length must have loops, which can be removed. Therefore, by selecting actions non-deterministically, the solution can be found in at most $2^{2^{|\mathcal{F}|}}$ steps. Hence, the plan existence problem for $\mathcal{P_{CO}}$ is in NEXPSPACE. By Savitch's theorem \cite{savitch1970relationships}, NEXPSPACE = EXPSPACE. Therefore, the plan existence problem for $\mathcal{P_{CO}}$ is EXPSPACE-complete. 
\end{proof}

\subsection{Algorithm for Plan Computation}
\label{section:algo}

\begin{algorithm}[!t]
\scriptsize
\SetAlgoLined
\KwIn{$\mathcal{P_{CO}} = \langle \mathcal{D}, \mathcal{G}, \Omega, \mathcal{O} \rangle$}
\KwOut{plan solution $\pi_{\mathcal{P_{CO}}}$, observation sequence, $O_{\mathcal{P_{CO}}}$}

$\Delta \gets 1 $ \Comment*[f]{Counter} \\
$\Delta\_limit \gets False$ \Comment*[f]{Delta cardinality flag} \\

\While{$\Delta \leqslant |\mathcal{S}|$}{ \label{line:outer}

$s_{\Delta} \gets \{ \mathcal{I} \}$ \Comment*[f]{Initial state}\\
$open \gets \texttt{Priority\_Queue()}$ \Comment*[f]{Open list} \\
$closed \gets \{\}$ \Comment*[f]{Closed list} \\
$b_0 \gets \{ \mathcal{O}(\emptyset, s_{\Delta}) \}$ \Comment*[f]{Initial belief} \\
$open.push(\langle \mathcal{I}, b_0 \rangle, priority = 0 )$ \\

\If{$ |s_{\Delta}| = \Delta $}{
$\Delta\_limit \gets True$ \\
}

\While{$open \neq \emptyset$ }{ \label{line:inner}
$\langle s_{\Delta}, b \rangle \gets open.pop() $\\
\If{$ \neg \Delta\_limit$}{
\For{$\hat{s} \in b \setminus s_{\Delta} $ }{
$s_{\Delta} \gets s_{\Delta} \cup \hat{s}$ \\
\If{$ |s_{\Delta}| = \Delta $}{
$\Delta\_limit \gets True$ \\
$\textbf{break}$
}
}
}
$closed \gets closed \cup s_{\Delta} $ \\
\If{ $\langle s_{\Delta}, b \rangle \models$ \textsc{GOAL-TEST}($\mathcal{G}$)}{ \label{line:goal} 
$\textbf{return}~\pi_{\mathcal{P_{CO}}}, O_{\mathcal{P_{CO}}}$
}

\For{$s_{\Delta}^{\prime} \in successors(s_{\Delta})$}{
$o \gets \mathcal{O}(a, s_{\Delta}^{\prime})$ \\
$b^{\prime} \gets$ Belief-Generation($b, a, o$) \\
$h(s_{\Delta}^{\prime}) \gets$ \textsc{HEURISTIC-FUNCTION}$(s_{\Delta}^{\prime}, b^{\prime})$ \label{line:heuristic} \\
\uIf{$s_{\Delta}^{\prime} \notin$ open \textbf{and} $s_{\Delta}^{\prime} \notin$ closed}{
$open.push(\langle s_{\Delta}^{\prime}, b^{\prime} \rangle , h(s_{\Delta}^{\prime}))$ \\
} 
\ElseIf{$h(s_{\Delta}^{\prime}) < h^{prev}(s_{\Delta}^{\prime})$ }{
\uIf{$s_{\Delta}^{\prime} \notin$ open}{
$closed \gets  closed \setminus s_{\Delta}^{\prime} $\\
$open.push(\langle s_{\Delta}^{\prime}, b^{\prime} \rangle , h(s_{\Delta}^{\prime})$) \\
}
\Else{
update priority from $h^{prev}(s_{\Delta}^{\prime})$ to $h(s_{\Delta}^{\prime})$ 
}
}
}
}
$\Delta \gets \Delta + 1$ \\
$\Delta\_limit \gets False$ \\
}

\textbf{procedure} Belief-Generation($b$, $a$, $o$) \\
$b^{\prime} \gets \{ \}$ \\
\For{$\hat{s} \in b$}{
\For{$\hat{a} \in \mathcal{A}$}{
\If{$\hat{s} \models pre(\hat{a})$ \textbf{and} $\mathcal{O}(\hat{a}, \Gamma(\hat{s}, \hat{a})) = o $}{
$b^{\prime} \gets  b^{\prime} \cup \Gamma(\hat{s}, \hat{a}) $ \\
}
}
}
\textbf{return} $b^{\prime}$
\caption{Plan Computation}
\label{procedure:belief}
\end{algorithm}

We present the details of a common algorithm template used by our formulations in Algorithm \ref{procedure:belief}. In Section \ref{section:compute}, we show how we customize the goal-test (line \ref{line:goal}) and the heuristic function (line \ref{line:heuristic}) to suit the needs of each of our problem variants. There are two loops in the algorithm: the outer loop (line \ref{line:outer}) runs for different values of $\Delta = \{ 1, 2, \ldots, |\mathcal{S}|\}$; while the inner loop (line \ref{line:inner}) performs search over the state space of size ${|\mathcal{S}| \choose \Delta}$. These loops ensure the complete exploration of the belief space. 

For each outer iteration, $s_{\Delta}$ is augmented with elements of the belief state until the cardinality of $s_{\Delta}$ is equal to the value of $\Delta$. In the inner loop, we run GBFS over the state space of $s_{\Delta}$. For each successor node in the open list, the belief induced by an observation is updated. The heuristic value of a state is computed using a plan graph \cite{blum1997fast} level based heuristic, such as set-level heuristic \cite{nguyen2002planning}. The plan graph data structure contains information about the positive and the negative interactions between the sets of propositions and actions. We use set-level plan graph heuristic to guide the search. To get the set-level cost, the plan graph is populated with a state, $s$ (search node), and it is expanded until one of the following holds $\bf{(1)}$ the goal is reachable, that is, the goal propositions are present in a proposition layer and are mutex-free pairs, or $\bf{(2)}$ the graph levels off, that is, it cannot be expanded further. If the goal is not reachable before the graph levels off then it cannot be achieved by any plan. In this case, the heuristic cost is $\infty$. Else, when the goal is reachable and the goal propositions are pairwise mutex-free, the heuristic value is the index of the first plan graph layer that contains it. 

\begin{prop}
Algorithm \ref{procedure:belief} necessarily terminates in finite number of $|\mathcal{S}|$ iterations, such that, the following conditions hold: \\

\noindent (\textbf{Completeness}) Algorithm \ref{procedure:belief} explores the complete solution space of $\mathcal{P_{CO}}$, that is, if there exists a $\pi_{\mathcal{P_{CO}}}$ that correctly solves $\mathcal{P_{CO}}$, it will be found. \\

\noindent (\textbf{Soundness}) The plan, $\pi_{\mathcal{P_{CO}}}$, found by Algorithm \ref{procedure:belief} correctly solves $\mathcal{P_{CO}}$ as ensured by the corresponding goal-test.
\end{prop}

Algorithm \ref{procedure:belief} terminates either when a plan is found or after running the outer loop for $|\mathcal{S}|$ iterations. The outer loop ensures that the all the paths in the search space are explored. And the goal tests of both of the problem variants ensure that the solutions are correct with respect to Definitions \ref{def:obf} and \ref{def:leg}. 

The increase in cardinality of $s_{\Delta}$ can lead to increase in the search overhead. In our implementation, we run only the first iteration of the outer loop. Most of the problem instances can be solved in the first iteration itself. 



\section{Plan Computation}
\label{section:compute}
In this section, we present instantiations of modules presented in Algorithm \ref{procedure:belief} for goal obfuscation and legibility. 

\subsection{Computing Goal Obfuscated Plans}

\paragraph{Goal test}

We ensure that the solution plan does not end unless all $k$ goals occur in the belief state. In order to achieve this, the goal condition checks whether the goal has been achieved in the agent's true state and also whether the $k-1$ goals have been achieved in the belief state. If there doesn't exist a plan that can achieve the true goal with the given $k-1$ goals, we restart the search with next combination of $k-1$ goals among the total $n-1$ decoy goals. 

\paragraph{Heuristic function}

We now propose a heuristic strategy for generating a plan where the last belief state satisfies \textit{k} goals, making it a \textit{k-ambiguous} plan. We use two heuristic functions to estimate a node's cost: 
\begin{equation}
h(s) = \left( \textit{set-level}_{G_A}(s) + \textit{set-level}_{\mathcal{G}_{k-1}}(b) \right)
\end{equation}
where the first heuristic computes the \textit{setLevel} distance to the true goal from the agent's actual state, while the second heuristic computes the \textit{setLevel} distance to $k-1$ goals from the belief induced by the emitted observation. This is computed by finding the max of the minimum set-level distance from belief to each of the $k-1$ goals. The heuristic value of a state is computed as the total of these two values. The heuristic ensures at least $k$ goals occur in the last belief induced by the plan. 

\paragraph{Note on deterministic output and use of noop}

The \textit{k-ambiguous} algorithm maintains obfuscation assuming that the adversarial observer does not have access to the process of plan computation. If the observer has access to it then, a simple variation like the addition of random noise to the heuristic can thwart the attack.

Our formulation supports the use of noops for obfuscation. A noop action can be compiled to multiple noop actions, $noop = \{pretend_{o_{i}} | i = 1, \ldots, |\Omega| \}$ where each noop corresponds to an observation of a prior action. This would allow the agent to emit observations without doing any action. This way the agent can increase the chances of obfuscation while using the same approaches. 

\subsection{Computing Goal Legible Plans}

\paragraph{Goal test} 

In this case, we have to ensure that the computed plan is consistent with \textit{at most} $j$ true goals. To that end, we change our goal condition to additionally check whether at most $j-1$ confounding goals have been achieved in the last belief in $BS$, and whether at least $n-j$ goals are absent in the belief.

\paragraph{Heuristic function}

In this case, our objective is to avoid at least $n-j$ goals and at the same time be consistent with at most $j$ goals. We achieve this by minimizing the set-level distance to true goal from the agent's actual state and to the $j-1$ confounding goals from the agent's belief state. However, we maximize the set-level distance to other $n-j$ goals in order to achieve at most $j$ goals. This is written as,
\begin{equation}
h(s) = \left( \textit{set-level}_{G_A}(s) + \textit{set-level}_{\mathcal{G}_{j-1}}(b) - \textit{set-level}_{\mathcal{G}_{n-j}}(b) \right)
\end{equation}

\section{Plan Obfuscation and Plan Legibility}

We also discuss two other problem variants called plan obfuscation and plan legibility. These problems come into picture when the observer is aware of the agent's goal. However, the objective here is to either obfuscate or reveal the next steps in the plan. 

\subsection{Obfuscation}

We achieve plan obfuscation by computing a plan whose observation sequence conforms to a set of diverse plans, making it hard to predict the next action in that plan.

\begin{dfn} 
A \textbf{plan obfuscation planning problem} is a tuple, $\mathcal{P}_{PO} = \langle \mathcal{D}, \mathcal{G}_{PO}, \Omega, \mathcal{O} \rangle$, where, $\mathcal{G}_{PO} = \{ G_A \}$, and $G_A$ is the true goal of the agent. 
\end{dfn}

The solution to a plan obfuscation planning problem is an $\ell$\emph{-diverse} plan. An $\ell$\emph{-diverse} plan has an observation sequence that is consistent with $\ell$ plans that are at least $d$ distance away. In order to compute an $\ell$\emph{-diverse} plan, we need to keep track of the paths that are consistent with the belief sequence of the given plan, we call the set of these paths as belief plan set. 

\begin{dfn} A \textbf{belief plan set}, BPS(p, $s_0$) = $\{ p_1, \ldots, p_n \}$, induced by a plan $p$ starting at $s_0$, is a set of plans that are formed by causally consistent chaining of state sequences in $BS(p, s_0)$, i.e., BPS(p, $s_0$) = $\{ \langle \hat{s}_0, \hat{a}_1, \hat{s}_1, \ldots, \hat{s}_n \rangle ~|~ \exists~ \hat{a}_j, ~\hat{s}_{j-1} \models pre(\hat{a}_j) ~\wedge~ \hat{s}_{j-1} \in b_{j-1} ~\wedge~ \hat{s}_{j} \models \hat{s}_{j-1} \cup add(\hat{a}_j) \setminus delete(\hat{a}_j) ~\wedge~ \hat{s}_{j} \in b_{j} \}$.
\end{dfn}

Our aim is to compute the diversity between all the pairs of plans in $BPS(p, s_0)$. The diversity between plans can be enforced by using plan distance measures. 

\subsubsection{Plan Distance Measures}
We will utilize the three plan distance measures introduced in \citet{srivastava2007domain}, and refined in \citet{nguyen-partialp-2012}, namely action, causal link and state sequence distances. Our aim is to use these plan distance measures to measure the diversity of plans in a belief plan set. 

\paragraph{Action distance}

We denote the set of unique actions in a plan $\pi$ as $A(\pi) = \{a~|~a\in\pi\}$. Given the action sets $A(p_1)$ and $A(p_2)$ of two plans $p_1$ and $p_2$ respectively, the action distance is,
\begin{math} 
\delta_A(p_1, p_2) = 1 - \frac{\lvert A(p_1) \cap A(p_2) \rvert}{\lvert A(p_1) \cup A(p_2) \rvert}
\end{math}. 

\paragraph{Causal link distance}

A causal link represents a tuple of the form $\langle a_i, p_i, a_{i+1} \rangle$, where $p_{i}$ is a predicate that is produced as an effect of action $a_i$ and used as a precondition for $a_{i+1}$. The causal link distance for the causal link sets $Cl(p_1)$ and $Cl(p_2)$ of plans $p_1$ and $p_2$ is,
\begin{math} \label{eqn:eqn2}
\delta_C(p_1, p_2) = 1 - \frac{\lvert Cl(p_1) \cap Cl(p_2) \rvert}{\lvert Cl(p_1) \cup Cl(p_2) \rvert}
\end{math}. 

\paragraph{State sequence distance}

This distance measure takes the sequences of the states into consideration.
Given two state sequence sets $S(p_1) = (s^{p_1}_0, \ldots, s^{p_1}_n)$ and $S(p_2) = (s^{p_2}_0, \ldots, s^{p_3}_{n^{\prime}})$ for $p_1$ and $p_2$ respectively, where $n \geq n^{\prime}$ are the lengths of the plans, the state sequence distance is,
\begin{math} \label{eqn:eqn3}
\delta_S(p_1, p_2) = \frac{1}{n}  \big[ \  \sum_{k=1}^{n^{\prime}} d(s_k^{p_1}, s_k^{p_2}) + n - n^{\prime} \big] \ \
\end{math},
where $d(s_k^{p_1}, s_k^{p_2}) = 1 - \frac{\lvert s_k^{p_1} \cap s_k^{p_2} \rvert}{\lvert s_k^{p_1} \cup s_k^{p_2} \rvert}$ represents the distance between two states (where $s_k^{p_1}$ is overloaded to denote the set of fluents in state $s_k^{p_1}$).

\noindent We now formally define $\ell$\emph{-diverse} plan and other terms. 

\begin{dfn} Two plans, $p_1, p_2$, are a \textbf{d-distant pair} with respect to distance function $\delta$ if, $\delta(p_1, p_2) = d$, where $\delta$ is a diversity measure. 
\end{dfn}


\begin{dfn} A BPS induced by plan p starting at $s_0$ is \textbf{minimally d-distant}, $d_{min}(BPS(p, s_0))$, if $\displaystyle d = \min_{p1, p2 \in BPS(p, s_0)} \delta(p1, p2)$.
\end{dfn}


\begin{dfn} A plan, $\pi_l$, is an \textbf{$\bm{\ell}$-diverse plan}, if for a given value of d and distance function $\delta$, $d_{min}(BPS(\pi_l, \mathcal{I})) \geq d$, $|BPS(\pi_l, \mathcal{I})| \geq \ell$, where $\ell \geq 2$ and every plan in $BPS(\pi_l, \mathcal{I})$ achieves the goal in $\mathcal{G}_{PO}$.
\end{dfn}


\subsubsection{Computing Obfuscated Plans}

Here we return a plan that is at least $\ell$\textit{-diverse} and that maximizes the plan distance between \textit{BPS} induced by a plan.

\paragraph{Goal test} 

To ensure the plans in $BPS$ induced by $\ell$\textit{-diverse} plan can achieve the goal in $\mathcal{G_{PO}}$, we change the goal condition to additionally check whether at least $\ell$ plans are reaching the goal or not. Also in order to ensure termination of the algorithm, there is a cost-bound given as input to the algorithm.

\paragraph{Heuristic function}

We now present our heuristic strategy to compute $\ell$\emph{-diverse} observation sequence. Our heuristic is a three-part function: 
\begin{equation}
h(s) = \left(- d_{min}(BPS(p, s_0)), -|BPS(p, s_0)|, \textit{set-level}(s) \right) 
\end{equation}
where the primary heuristic maximizes the $d$ of $d_{min}(BPS(p, s_0))$ induced by plan $p$ starting at $s_0$, the second heuristic maximizes the cardinality of the set $BPS(p, s_0)$, while the third heuristic gives the set-level value of $s$. The cardinality of $BPS(p, s_0)$ is computed by counting the number of plans that have the same set-level value from their last states as that from the true state. The distance is computed using either of the three plan distances.  

\subsection{Plan Legibility}

The definition of plan legibility planning problem is similar to that of plan obfuscation planning problem. Here the objective is to reveal the next action in the plan. We achieve this by making the observation sequence consistent with plans that are similar in terms of observation sequences. This case is opposite to that of the plan obfuscation case. We use the same three plan distance measures to define the similarity of the plans. We call the solution to this problem as an \emph{m-similar} plan. For similar plans in $BPS$, we define the maximum distance between any two pairs of plans.

\begin{dfn}
A BPS induced by plan p starting at $s_0$ is \textbf{maximally d-distant}, $d_{max}(BPS(p, s_0))$, if $d = \displaystyle \max_{p1, p2 \in BPS(p, s_0)} \delta(p1, p2)$.
\end{dfn}

\begin{dfn} A plan, $\pi_m$, is a \textbf{m-similar plan}, if for a given value of d and distance function $\delta$, $d_{max}(BPS(\pi_m, \mathcal{I})) \leq d$, $|BPS(\pi_m, \mathcal{I})| \geq m$, where $m \geq 2$ and every plan in $BPS(\pi_m, \mathcal{I})$ achieves the goal in $\mathcal{G}_{PL}$.
\end{dfn}


\subsubsection{Computing Legible Plans}

Here we return a plan that is at least $m$\textit{-similar} and that minimizes the plan distance between \textit{BPS} induced by a plan. The goal test for this case is similar to that of plan obfuscation case. 

\paragraph{Heuristic function}

The heuristic function to compute $m$\emph{-similar} observation sequence is a three-part function: 
\begin{equation}
h(s) = \left(d_{max}(BPS(p, s_0)), -|BPS(p, s_0)|, \textit{set-level}(s) \right) 
\end{equation}
where the primary heuristic minimizes the $d$ value of $d_{max}(BPS(p, s_0))$ induced by plan $p$ starting at $s_0$, the remaining parts are same as those in plan obfuscation case.

\section{Robot Demonstration}

We now present a demonstration of \textit{k-ambiguity} and \textit{j-legibility} using a physical Fetch robot in a cooking domain. Here robot plays the role of an actor and a human plays the role of an observer. The goal of the robot is to add necessary ingredients to make a recipe. In the obfuscation setting, the robot's aim is to hide the ingredients from the human. This can be of use, especially, in scenarios where there is a need to preserve the secrecy of a recipe due to proprietary reasons. In contrast, in the legibility setting, the robot's aim is to inform the human about the ingredients being added. This could be seen from a household user's perspective, where the user needs to know what is being used in the recipe. 

The domain consists of actions that allow the robot to communicate with the human, as well as pick-up and pour ingredients into a container. Each action can be performed in either an obfuscated or a legible way with different effects and different observations. We provide the action descriptions in Figures \ref{fig:desc1} and \ref{fig:desc2}. From Figure \ref{fig:desc1}, it can be seen that the action \texttt{ask-human-to-stir} distracts the human and allows the robot to obfuscate its goal. When the human is inattentive, the observation emitted for picking up any ingredient is generic. On the other hand, in Figure \ref{fig:desc2} the action \texttt{ask-human-to-pass-sugar} brings the human to attention. When the human is attentive, the observation emitted for picking up a container is specific for each ingredient. As shown in Figure \ref{fig:scenario}, the scenario consists of a salt container and two sugar containers: one labeled as sugar, another without a label. The human is not aware of the contents of the unlabeled container. 

\begin{figure}
\centering
\includegraphics[width=0.8\columnwidth]{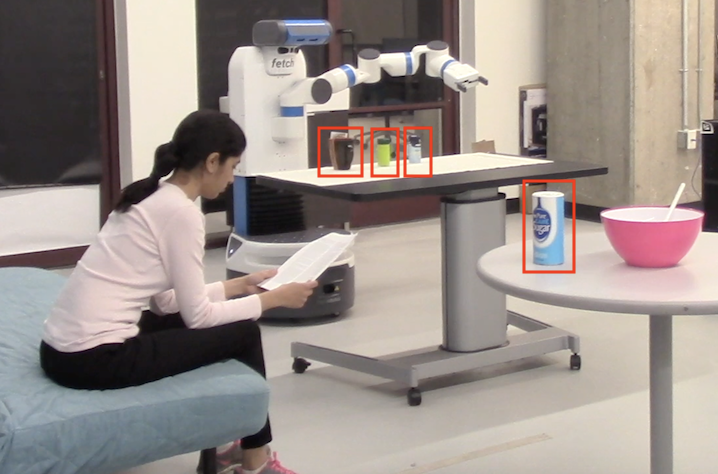}
\caption{An illustration of the cooking domain. Here the robot is following a recipe and has to pick up and pour the ingredients into a cup. The robot has two containers available on its table, an unlabeled green colored container that contains sugar (unbeknownst to the human) and a labeled salt container. A labeled sugar container is available on the second table.}
\label{fig:scenario}
\end{figure}

\begin{figure}[t]
\centering
\begin{tabular}{|l|}
\hline
(:action ask-human-to-stir\\
\tab :parameters (?x - ingredient ?y - container)\\
\tab :precondition (and \\
\tab \tab (in ?x ?y) \\
\tab \tab(not (obfuscated ?y)) \\
\tab ) \\
\tab :effect (and \\
\tab \tab (human-inattentive) \\
\tab \tab (increase (total-cost) 1) \\
\tab ) \\
 ) \\
\\
(:action pickup-container-obfuscated \\
\tab :parameters (?x - ingredient ?y - container \\
\tab ?z - table) \\
\tab :precondition (and \\
\tab \tab (in ?x ?y) \\
\tab \tab (handempty) \\
\tab \tab (on ?y ?z) \\
\tab \tab (accessible ?z) \\
\tab \tab (human-inattentive) \\
\tab ) \\
\tab :effect (and \\
\tab \tab (not (handempty)) \\
\tab \tab (holding ?y) \\
\tab \tab (not (on ?y ?z)) \\
\tab \tab (obfuscated ?y) \\
\tab \tab (increase (total-cost) 1) \\
\tab ) \\
) \\
\hline
\end{tabular}
\caption{Action description of obfuscating actions}
\label{fig:desc1}
\end{figure}

\begin{figure}[t]
\centering
\begin{tabular}{|l|}
\hline
(:action ask-human-to-pass-sugar \\
\tab :parameters (?x - ingredient ?y - container \\ \tab ?z - table) \\
\tab :precondition (and \\
\tab \tab (in ?x ?y) \\
\tab \tab (is-type-labeled ?y) \\
\tab \tab (accessible ?z) \\
\tab \tab (not (legible ?y)) \\
\tab ) \\
\tab :effect (and  \\
\tab \tab (on ?y ?z)  \\
\tab \tab (human-attentive) \\
\tab \tab (increase (total-cost) 1) \\
\tab ) \\
) \\
\\
    
(:action pickup-container-legible \\
\tab :parameters (?x - ingredient ?y - container \\
\tab ?z - table) \\
\tab :precondition (and \\
\tab \tab (in ?x ?y) \\
\tab \tab (handempty) \\
\tab \tab (on ?y ?z) \\
\tab \tab (accessible ?z) \\
\tab \tab (is-type-labeled ?y) \\
\tab \tab (human-attentive) \\
\tab ) \\
\tab :effect (and \\
\tab \tab (not (handempty)) \\
\tab \tab (holding ?y) \\
\tab \tab (not (on ?y ?z)) \\
\tab \tab (legible ?y) \\
\tab \tab (increase (total-cost) 1) \\
\tab ) \\
) \\
\hline
\end{tabular}
\caption{Action description of legible actions}
\label{fig:desc2}
\end{figure}

We used MoveIt \cite{sucan2013moveit} software with \texttt{RRTConnectkConfigDefault} planner to perform motion planning. RViz was used to visualize the objects in the scene and the possible trajectories of the robot. AR markers were used to localize the objects in the scene. The candidate goal set along with \textit{k-ambiguous} and \textit{j-legible} plans are given in Table \ref{table:robot}. Here $k=2$ and $j=1$. The videos for both obfuscation and legibility can be viewed at \protect{\url{https://youtu.be/mW7qCeggpUA}} and \protect{\url{https://youtu.be/qzvTivKuszU}}.

\setlength{\tabcolsep}{4pt}
\renewcommand{\arraystretch}{1.1}
\begin{table}
\caption{Obfuscated and legible plans for cooking domain}
\resizebox{\columnwidth}{!}{%
\begin{tabular}{ |l|l| }
\hline
True Goal & contents-added cup1 sugar \\
\hline
Decoy / Confounding & contents-added cup1 salt  \\
\hline
Obfuscated Plan & ask-human-to-stir sugar container1, 
 \\
 & pickup-container-obfuscated sugar container1 table1, \\
 & pour-ingredient-and-place sugar container1 cup1 table1 \\
\hline
Obfuscation observation & human-inattentive, pickup-ingredient, pour-ingredient \\
\hline
Legible Plan & ask-human-to-pass-sugar sugar container2 table1,  \\
& pickup-container-legible sugar container2 table1,  \\
& pour-ingredient-and-place sugar container2 cup1 table1 \\
\hline
Legible observation & human-attentive, pickup-sugar, pour-sugar\\
\hline
\end{tabular}%
}
\label{table:robot}
\end{table}

\section{Empirical Evaluation}
\label{sec:evaluation}
We now present an empirical analysis of all four approaches to demonstrate the performance and scope of our system. To this end, we use the following three IPC domains for our experiments: 15 problems from \texttt{BlocksWorld}, 10 from \texttt{Storage}, 10 from \texttt{Gripper}. We use grounded domains and problems in our experiments. 

\subsection{Observation Model Setup}

We now discuss the design of $\mathcal{O}$ for the following domains:

\subsubsection{BlocksWorld}  
For blocksworld, the observation model is same as the one defined in Example \ref{example:1}. For the 15 problems used in this domain, the smallest problems had 26 fluents, 34 actions and the largest problems had 39 fluents, 74 actions.  


\subsubsection{Storage} The storage domain has actions \texttt{move}, \texttt{drop}, \texttt{lift}, \texttt{go-out} and \texttt{go-in}. The observer cannot distinguish between lift and drop actions but receives an observation saying whether the hoist was used to perform an action. The observer can tell whether the agent is inside a particular storage area or outside. However, once the agent is inside a store area, the move actions do not reveal the agent's area. Therefore all move actions are of the same type. For the 10  test problems, the smallest problem had 12 fluents, 10 actions and the largest problem had 43 fluents, 78 actions.  

\subsubsection{Gripper} The gripper domain has actions \texttt{move}, \texttt{drop} and \texttt{pickup}. The observer gets observation when the agent moves from one room to another. Also, the observer gets an observation regarding whether the gripper is holding something or not. Therefore in this domain, the observer can distinguish between all types of actions. However, the observer is not aware of the exact location of the agent. For the 10 problems used in this domain, the smallest problem had 21 fluents, 30 actions and the largest problem had 40 fluents, 96 actions after grounding. 

\setlength{\tabcolsep}{5pt}
\renewcommand{\arraystretch}{1.1}
\begin{table}[!t]
\resizebox{\columnwidth}{!}{%
\begin{tabular}{ |l|c|c|c|c|c|c| }
\hline
\multirow{2}{*}{Domain} & \multirow{2}{*}{\small{Metrics}} & \multirow{2}{*}{\textit{k-amb}} & $\ell$\textit{-div} & $\ell$\textit{-div}& $\ell$\textit{-div}  \\
&& & (action) &(causal)& (state) \\
\hline
\hline
\multirow{3}{*}{Blocksworld} & avg time & 32.20 & 123.41 & 174.06 & 571.03  \\ \cline{2-6}
& sd time & 82.15 & 155.72 & 210.49 & 169.37 \\ \cline{2-6}
& $|O|$ & 9.33 & 7.71 & 6.85 & 7.11 \\
\hline
\multirow{3}{*}{Storage} &avg time & 37.72 & 88.11 & 212.49 & 227.58 \\ \cline{2-6}
& sd time & 35.80 & 90.38 & 374.14 & 250.79 \\ \cline{2-6}
& $|O|$ &7.83 & 6.75 & 5.83 & 5.66 \\
\hline
\multirow{3}{*}{Gripper} & avg time & 56.49 & 175.56 & 592.94 & 149.63 \\ \cline{2-6}
&sd time & 118.64 & 52.41 & 197.61 & 48.87 \\ \cline{2-6}
&$|O|$ & 6.88 & 4.3 & 5.12 & 4.55 \\
\hline
\end{tabular}%
}
\caption{Empirical evaluation for goal obfuscation and plan obfuscation. We report average, standard deviation of time taken in seconds and the average observation sequence length of the obfuscated plans.}
\label{table:eval1}
\end{table}

\setlength{\tabcolsep}{5pt}
\renewcommand{\arraystretch}{1.1}
\begin{table}[!t]
\resizebox{\columnwidth}{!}{%
\begin{tabular}{ |l|c|c|c|c|c|c|c| }
\hline
\multirow{2}{*}{Domain} & \multirow{2}{*}{\small{Metrics}} & \multirow{2}{*}{\textit{j-leg}} & \textit{m-sim} & \textit{m-sim}& \textit{m-sim} \\
&& & (action) &(causal)& (state) \\
\hline
\hline
\multirow{3}{*}{Blocksworld} & avg time & 204.12 & 59.63 & 73.56& 81.07\\ \cline{2-6}
& sd time & 155.04 &73.21 &88.03& 127.62\\ \cline{2-6}
& $|O|$ & 6.9 &6.93 &7.14 & 6.85\\
\hline
\multirow{3}{*}{Storage} &avg time & 14.21 &36.34& 31.97& 38.79\\ \cline{2-6}
& sd time & 15.65 &41.52&27.50& 52.09\\ \cline{2-6}
& $|O|$ &5.27 &9.8&9.66& 10.12 \\
\hline
\multirow{3}{*}{Gripper} & avg time & 383.17 & 329.37& 314.62&349.66 \\ \cline{2-6}
&sd time & 178.14 & 131.70 & 112.64& 159.65\\ \cline{2-6}
&$|O|$ & 6.75 & 7.34 &8.62& 8.33\\
\hline
\end{tabular}%
}
\caption{Empirical evaluation for goal legibility and plan legibility. We report average, standard deviation of time taken in seconds and the average observation sequence length of the legible plans.}
\label{table:eval2}
\end{table}

\subsection{Results}

We provide evaluation of our approaches in Table \ref{table:eval1} and \ref{table:eval2}. We wrote new planners from scratch for each of the the algorithms presented. We ran our experiments on Intel(R) Xeon(R) CPU E5-2643v3, with a time out of 30 minutes. We created the planning problems in a randomized fashion. We report the performance of our approaches in terms of average and standard deviation for the time taken to run the problems in the given domain, and the average length of the observation sequence. For all the problems, the values used were $k = 5$, $\ell = 3$, $j = 3$ with $n-j = 2$, $m =3$, $d_{min} = 0.25$ and $d_{max} = 0.50$. 

\setlength{\tabcolsep}{4pt}
\renewcommand{\arraystretch}{1.1}
\begin{table*}[!thp]
\resizebox{\textwidth}{!}{%
\begin{tabular}{ |l|l|l| }
\hline
 Algo, $\mathcal{O} $ & Plan & Observation Sequence  \\
\hline
\hline
FD,\{$\mathcal{O}_1, \mathcal{O}_2\}$ & unstack-B-C, putdown-B, unstack-C-A, putdown-C, unstack-A-D, stack-A-B  & unstack, putdown, unstack, putdown, unstack, stack \\
\hline
\multirow{2}{*}{k-amb, $\mathcal{O}_1$} & unstack-B-C, putdown-B, unstack-C-A, putdown-C, unstack-A-D, stack-A-B, & unstack, putdown, unstack, putdown, unstack, stack, pickup,  putdown, \\
& pickup-C, putdown-C, pickup-D, putdown-D, pickup-C, stack-C-D &  pickup, putdown, pickup, stack \\
\hline
\multirow{2}{*}{k-amb, $\mathcal{O}_2$} & unstack-B-C, putdown-B, unstack-C-A, putdown-C, unstack-A-D, stack-A-B,  & unstack-B, putdown-B, unstack-C, putdown-C, unstack-A, stack-A,  \\
&unstack-A-B, putdown-A, pickup-B, stack-B-C, pickup-A, stack-A-B & unstack-A, putdown-A, pickup-B, stack-B, pickup-A, stack-A \\
\hline
\multirow{2}{*}{$\ell$-div, $\mathcal{O}_1$} & unstack-B-C, putdown-B, unstack-C-A, stack-C-B, unstack-C-B, putdown-C, & \multirow{2}{*}{unstack, putdown, unstack, stack, unstack, putdown, unstack, stack}  \\
&  unstack-A-D, stack-A-B & \\
\hline
\multirow{2}{*}{$\ell$-div, $\mathcal{O}_2$} & unstack-B-C, putdown-B, unstack-C-A, stack-C-B, unstack-C-B, putdown-C, & unstack-B, putdown-B, unstack-C, stack-C, unstack-C, putdown-C,  \\
& unstack-A-D, stack-A-B & unstack-A, stack-A\\
\hline
\multirow{2}{*}{j-leg, $\mathcal{O}_1$}  & unstack-B-C, putdown-B, unstack-C-A, putdown-C, pickup-B, stack-B-C,  & \multirow{2}{*}{unstack, putdown, unstack, putdown, pickup, stack, unstack, stack} \\
& unstack-A-D, stack-A-B &  \\
\hline
\multirow{2}{*}{j-leg, $\mathcal{O}_2$}  & unstack-B-C, putdown-B, unstack-C-A, putdown-C, pickup-B, stack-B-C,  & unstack-B, putdown-B, unstack-C, putdown-C,  pickup-B, stack-B,  \\
& unstack-A-D, stack-A-B & unstack-A, stack-A\\
\hline
\multirow{2}{*}{m-sim, $\mathcal{O}_1$} & unstack-B-C, putdown-B, unstack-C-A, putdown-C, unstack-A-D, & \multirow{2}{*}{unstack, putdown, unstack, putdown, unstack, stack}  \\
&  stack-A-B & \\
\hline
\multirow{2}{*}{m-sim, $\mathcal{O}_2$} & unstack-B-C, putdown-B, unstack-C-A, putdown-C, unstack-A-D, putdown-A & unstack-B, putdown-B, unstack-C, putdown-C, unstack-A, putdown-A,  \\
& pickup-A, stack-A-B & pickup-A, stack-A\\
\hline
\end{tabular}%
}
\caption{Examples of plans generated for two different observation models}
\label{table:eval3}
\end{table*}


For all the domains, we note that the average length of \textit{k-ambiguous} observation sequences is longer than that of \textit{j-legible} observation sequences as obfuscated plans involve incorporating ``time-wasting" actions. The computation of legible plans, in general, is a challenging task, as it is difficult to inform an observer about an agent's goals without any explicit communication. Moreover, the \textit{j-legible} case limits the number of conforming goals to at most $j$. 
For \textit{j-legible} case, $3/15$ problems in blocksworld and $1/10$ problems in gripper were unsolvable. We also note that the average observation length for all three distance measures is different, thus showing that they produce different plans. For $\ell$\textit{-diverse} and \textit{m-similar} case, $2/10$ problems in storage domain, were unsolvable for all three distance measures.  


\subsection{Example}

We illustrate the plans generated by our algorithms along with a default optimal plan generated by FD for two different observation models of \texttt{BlocksWorld}. $\mathcal{O}_1$ is same as the one discussed above with $|\Omega| = 4$. For $\mathcal{O}_2$, let's assume the agent can identify the block in its hand, then $|\Omega| = 16$. The initial state is, \textit{on-B-C, on-C-A, on-A-D, ontable-D, clear-B, handempty}. The true goal is \textit{on-A-B} and the decoy/confounding goals are \textit{on-B-C, on-D-C}. We compute plans for $k=3$, $\ell=2$ with action distance and $j=2$ in both the observation models. These plans along with accompanying observation sequences are given in Table \ref{table:eval3}. For the \textit{j-legible} case, $j$ consists of \textit{on-B-C} and true goal \textit{on-A-B}, while $n-j$ consists of \textit{on-D-C}. 
We note that the obfuscated plans consist of ``time-wasting" actions as can be seen in \textit{k-amb} and \textit{$\ell$-div} examples. We can see that the \textit{k-amb} and \textit{m-sim} plans are affected by the observation model difference. However, the other two are not affected, this is mostly because this is a small problem for illustrative purposes. In general, plan computation particularly depends on the observation model. 

\section{Related Work}

There are prior works which discuss the problem of privacy preservation in distributed multi-agent systems \cite{brafman2015privacy,luis2014plan,bonisoli2014privacy}. A recent work on privacy for multi-agents of \citet{maliah2016stronger} is complementary to our approach, as they consider problems where the model needs to be protected from the team members but goals and behavior are coordinated. In contrast, we consider problems where the models are public but goals and behavior need to be protected.  

The problem of goal obfuscation is also related to plan recognition literature \cite{ramirez2009plan,ramirez2010plan,yolanda2015fast,sohrabi2016plan,keren2016goal}. Traditional plan recognition systems have focused on scenarios where actions being executed can be observed directly. In our case, observational equivalence due to the many-to-one formulation of $\mathcal{O}$ introduces, in effect, noisy action-state observations. This, in turn, complicates plan recognition. More crucially, the agent uses the observational equivalence to actively help or hinder the ease of plan recognition. 

There are a few recent works which have explored the idea of obfuscation in adversarial settings from the goal recognition aspect \cite{keren2016privacy,ijcai2017-610}. One of the closely related work is that of \citet{keren2016privacy} on privacy preservation, in which the authors propose a solution that obfuscates a goal by choosing one of the candidate goals that has the maximum non-distinct path in common with the true goal, which obfuscates part of the plan. In contrast, our plans are obfuscated for the entire length such that, at least $k$ goals are consistent with the observations. Also, our framework supports the case of plan obfuscation which prevents the next step from being deciphered by making it consistent with $\ell$ diverse plans, and the case of a cooperative observer which make the agent's intentions legible to the observer by being consistent with at most $j$ goals.

The notions of k-anonymity \cite{sweeney2002k} and l-diversity \cite{machanavajjhala2006diversity} were originally developed in the literature on privacy and security for relational databases. In motion planning and robotics community, legibility \cite{Dragan-RSS-13,knepper2017implicit} has been a well-studied topic. However, this has been mostly looked at from the motion planning perspective, and therefore the focus has been on optimizing the motion trajectories such that the goal is revealed. We borrow these notions and generalize it in a unified framework to provide obfuscated and legible plans from a task planning perspective.

\subsection{Compilation to Model Uncertainty}

In recent years, there has been some interesting research in the field of human aware planning. Especially the work on explainable AI and explanations \cite{fox2017explainable,exp-yz,explain} proposes modeling the human's understanding of a planning agent and introduces the notion of human-aware multi-model planning. Their framework consists of two models representing the planner's domain model and the observing or interacting human's understanding of the planning model. This setting captures the uncertainty of the observer in the form of human's partial or incorrect model of the agent. On the other hand, our setting also explores uncertainty of the observer's understanding of the plans computed by the planner. However, we capture the uncertainty in form of a partial observation model. We hypothesize that the two settings can be compiled from one formulation to another, and can be perceived as primal and dual problems. We intend to investigate this direction in future work.


\section{Conclusion}

We introduced a unified framework that gives a planner the capability of addressing both adversarial and cooperative situations. Our setting assumes that the observer has partial visibility of the agent's actions, but is aware of agent's planning capabilities. We define four problems: goal obfuscation and goal legibility when the agent's true goal is unknown and, plan obfuscation and plan legibility when the agent's true goal is known. We propose the following solutions to these problems: \emph{k-ambiguous} plan which obfuscates the true goal with respect to at least $k$ goals, \emph{j-legible} plan which enables an observer to quickly understand the $j$ true goals of the agent, $\ell$\emph{-diverse} plan which obfuscates the next actions in a plan and, \emph{m-similar} plan which reveals the next actions in the plan. We present different search techniques to achieve these solutions and evaluate the performance of our approaches using three IPC domains: \texttt{BlocksWorld}, \texttt{Storage} and \texttt{Gripper}. We also demonstrate the goal obfuscation and goal legibility problems using the Fetch robot in a cooking domain.

\section*{Acknowledgments}

This research is supported in part by the AFOSR grant FA9550-18-1-0067, the ONR grants N00014-16-1-2892, N00014-13-1-0176, N00014-13- 1-0519, N00014-15-1-2027, N00014-18-1-2442 and the NASA grant NNX17AD06G. 

\bibliographystyle{aaai}
\bibliography{bib}

\end{document}